\documentclass[12pt]{article}
\usepackage{amsmath,amssymb,amsthm,graphicx,cite}
\usepackage{geometry}
\usepackage{subcaption}

\newtheorem{theorem}{Theorem}
\title{Safety-Critical Control for Robotic Manipulators using Collision Cone Control Barrier Functions}
\author{Lucas Henrique Almeida\\University of S\~{a}o Paulo (USP)}
\date{}

\begin{document}

\maketitle

\begin{abstract}
This paper presents a comprehensive approach for the safety-critical control of robotic manipulators operating in dynamic environments. Building upon the framework of Control Barrier Functions (CBFs), we extend the collision cone methodology to formulate Collision Cone Control Barrier Functions (C3BFs) specifically tailored for manipulators. In our approach, safety constraints derived from collision cone geometry are seamlessly integrated with Cartesian impedance control to ensure compliant yet safe end-effector behavior. A Quadratic Program (QP)-based controller is developed to minimally modify the nominal control input to enforce safety. Extensive simulation experiments demonstrate the efficacy of the proposed method in various dynamic scenarios.
\end{abstract}

\section{Introduction}
The safety of robotic systems is an ever-increasing concern, particularly in applications where manipulators operate in environments populated by dynamic obstacles or interact closely with humans. Traditional trajectory tracking methods, which are designed to optimize performance metrics, may inadvertently lead to unsafe behaviors when unexpected obstacles are encountered. Many control strategies have been proposed to ensure system safety. These include methods such as artificial potential fields \cite{8022685,Singletary2021ComparativeAO}, reference governors \cite{6859176}, reachability analysis \cite{8263977,RA-UAV,https://doi.org/10.48550/arxiv.2106.13176}, and nonlinear model predictive control \cite{8442967,yu_mpc_aut_ground_vehicle}. However, to formally guarantee safety—specifically collision avoidance with obstacles—it is essential to utilize a safety-critical controller that places safety above tracking performance. In this context, Control Barrier Functions (CBFs) \cite{Ames_2017} have emerged as a promising approach. These functions define a safe set using inequality constraints and reformulate the safety condition as a quadratic programming problem to ensure that the system's state remains within the safe region over time. In parallel, the concept of collision cones \cite{doi:10.2514/1.G005879,709600,Fiorini1993}, which originated in the context of vehicle collision avoidance, has been successfully integrated with CBFs to yield Collision Cone Control Barrier Functions (C3BFs)\cite{tayal2024collision}. This approach constrains the relative velocity between the robotic agent and obstacles so that the latter remains outside a critical cone, thereby avoiding collisions. While previous work on C3BFs has primarily targeted unmanned aerial \cite{tayal2024control}, ground vehicles\cite{goswami2024collision} and legged robots \cite{tayal2023safe}, our work extends this methodology to robotic manipulators. In this paper, we integrate the collision cone-based safety constraints with Cartesian impedance control \cite{1242165}, a strategy commonly used for compliant manipulation, and derive a QP-based safety filter that modifies the nominal control input only when necessary. 
Although, there has been a lot of work done on using CBFs for manipulators \cite{9636794, WANG2022361, singletary2021safety}, but nothing has been done using C3BFs, which has shown to outperform state of the art Higher Order CBFs \cite{xiao2019control}.

The remainder of this paper is organized as follows. We begin by providing detailed background on manipulator dynamics, CBFs, and collision cone concepts. We then formulate the problem in the context of Cartesian impedance control and describe our methodology for synthesizing a safety-critical controller. Finally, we present simulation experiments that validate the proposed approach and conclude with discussions on the results and future work.

\section{Preliminaries}
In this section, we present the theoretical background necessary to understand the proposed safety-critical control framework. We first describe the dynamics of robotic manipulators and then review the concept of Control Barrier Functions (CBFs) as a means to enforce safety. Consider an $n$-degree-of-freedom (DOF) robotic manipulator characterized by the joint positions $q \in \mathbb{R}^n$, joint velocities $\dot{q} \in \mathbb{R}^n$, and joint accelerations $\ddot{q} \in \mathbb{R}^n$. The dynamics of the manipulator can be modeled by the Euler-Lagrange equation,
\begin{equation}
    M(q)\ddot{q} + C(q,\dot{q})\dot{q} + G(q) = \tau,
    \label{eq:dynamics}
\end{equation}
where $M(q)$ denotes the positive definite inertia matrix, $C(q,\dot{q})$ represents the Coriolis and centrifugal forces, $G(q)$ is the gravitational torque vector, and $\tau$ is the vector of control inputs (joint torques). This equation forms the basis for our subsequent control design. Control Barrier Functions are continuously differentiable functions $h(x)$ defined on the state space $x \in \mathbb{R}^n$, which in turn characterize a safe set $\mathcal{S} = \{ x \mid h(x) \geq 0 \}$. The core idea is that if the system trajectory starts within this safe set, then by enforcing the condition 
\begin{equation}
    \dot{h}(x) + \alpha(h(x)) \geq 0,
    \label{eq:cbf_condition}
\end{equation}
where $\alpha(\cdot)$ is an extended class-$\mathcal{K}$ function, the state will remain in $\mathcal{S}$ for all future times. For a control-affine system described by $\dot{x} = f(x) + g(x) u$, this inequality provides a constraint on the control input $u$ such that the safety of the system is guaranteed. 

The collision cone concept provides a geometrically intuitive means of characterizing potential collisions. For a manipulator, let $p_m \in \mathbb{R}^3$ denote the position of the end-effector and $p_o \in \mathbb{R}^3$ denote the position of an obstacle. The relative position is defined as $p_{rel} = p_m - p_o$, and the relative velocity is $v_{rel} = \dot{p}_m - \dot{p}_o$. The collision cone is then the set of all relative velocity vectors that, if maintained, would lead the end-effector to collide with the obstacle. We define a candidate Collision Cone Control Barrier Function (C3BF)\cite{tayal2024collision} as
\begin{equation}
    h(x,t) = \langle p_{rel}, v_{rel} \rangle + \|p_{rel}\|\|v_{rel}\|\cos\phi,
    \label{eq:c3bf_candidate}
\end{equation}
where $\phi$ is the half-angle of the collision cone determined by the geometry of the obstacle and the manipulator. The condition $h(x,t) \geq 0$ ensures that the relative velocity is directed away from the collision cone, thereby preventing a collision. In addition, Cartesian impedance control is employed to manage the interaction forces between the manipulator and its environment. The impedance control law is expressed as
\begin{equation}
    \Lambda(\ddot{x} - \ddot{x}_{des}) + D(\dot{x} - \dot{x}_{des}) + K(x - x_{des}) = f_{ext},
    \label{eq:impedance}
\end{equation}
where $x \in \mathbb{R}^3$ is the end-effector position, $\Lambda$, $D$, and $K$ denote the desired inertia, damping, and stiffness matrices respectively, $x_{des}$, $\dot{x}_{des}$, and $\ddot{x}_{des}$ represent the desired position, velocity, and acceleration, and $f_{ext}$ is the external force. The goal is to achieve compliant behavior while ensuring that the safety constraints derived from the collision cone approach are respected.

\section{Problem Formulation}
The objective of the proposed framework is to design a safety filter that integrates seamlessly with a Cartesian impedance controller in order to prevent collisions with dynamic obstacles while preserving the manipulator's tracking and compliant behaviors. Initially, a nominal control input is generated using Cartesian impedance control, which computes the desired acceleration and corresponding joint torques based on the error between the actual and desired end-effector trajectories. However, this nominal control input may not inherently guarantee safety when obstacles are present. To enforce safety, we define the safe set by the C3BF candidate given in Equation~\eqref{eq:c3bf_candidate}. Ensuring that the system remains in the safe set requires that the time derivative of the C3BF satisfies $\dot{h}(x,t) + \alpha(h(x,t)) \geq 0$. This safety condition is then incorporated into a Quadratic Program (QP) whose objective is to compute the final control input that deviates minimally from the nominal input while guaranteeing the satisfaction of the safety constraint. In essence, the final control command is obtained by solving
\begin{align}\label{eq:qp_formulation}
    \min_u & \quad \| u - u_{nom} \|^2 \\
    \text{s.t.} & \quad \dot{h}(x) + \alpha(h(x)) \geq 0.
\end{align}
By solving this QP in real time, the controller ensures that the manipulator’s end-effector avoids collision even when obstacles move unpredictably, while the nominal Cartesian impedance control ensures compliance and trajectory tracking.

\section{Methodology}
The methodology involves several key steps that culminate in the design of a safety-critical controller. First, the collision cone CBF candidate defined in Equation~\eqref{eq:c3bf_candidate} is derived by considering the relative position and velocity between the end-effector and an obstacle. The function is designed such that it remains positive when the relative velocity vector lies outside the collision cone, thereby ensuring safety. To utilize this candidate in control design, we differentiate it with respect to time. The derivative involves the term $\langle v_{rel}, v_{rel} \rangle$, as well as the derivative of the product $\|p_{rel}\|\|v_{rel}\|\cos\phi$, which captures the geometric evolution of the collision cone as both the manipulator and obstacle move. An appropriate extended class-$\mathcal{K}$ function, commonly selected as $\alpha(h) = \gamma h$ with $\gamma > 0$, is then applied to the derivative to enforce the safety condition in a manner that is proportional to the level of safety violation.

The second step is to integrate the safety constraint with the nominal Cartesian impedance control law given in Equation~\eqref{eq:impedance}. The impedance controller is responsible for generating a nominal control input that drives the end-effector towards a desired trajectory while ensuring compliant interaction with the environment. This nominal control input is typically derived using inverse dynamics or resolved motion rate control, converting task-space commands into joint-space torques. The third and final step involves the formulation of a Quadratic Program that minimally alters the nominal control input only when the safety constraint is active. The QP, as shown in Equation~\eqref{eq:qp_formulation}, is solved at each control cycle and returns the optimal control input that satisfies the safety constraint while remaining as close as possible to the nominal command. This approach guarantees that the safety filter is only activated when necessary, thereby preserving the original performance of the impedance controller under normal conditions.

\begin{theorem}[Validity of the Collision Cone CBF for a Manipulator End-Effector]
Consider a manipulator end‐effector whose dynamics are given by
\[
\dot{p}_m = v_m,\quad \dot{v}_m = u,
\]
where \(p_m\in\mathbb{R}^3\) denotes the end‐effector position, \(v_m\in\mathbb{R}^3\) its velocity, and \(u\in\mathbb{R}^3\) is the control input. Let an obstacle be moving with constant velocity \(v_o\) (i.e., \(\dot{p}_o=v_o\) and \(\dot{v}_o=0\)) with known position \(p_o\). Define the relative position and velocity as
\[
p_{rel} = p_m - p_o,\quad v_{rel} = v_m - v_o.
\]
Assume that the safety requirement is that the Euclidean distance satisfies \(\|p_{rel}\| > r\), where \(r>0\) is a given safety margin. Define the candidate Collision Cone Control Barrier Function (C3BF) as
\[
h(x,t)=\langle p_{rel}, v_{rel}\rangle + \|p_{rel}\|\|v_{rel}\|\cos\phi,
\]
where \(\cos\phi\) is defined as
\[
\cos\phi = \frac{\sqrt{\|p_{rel}\|^2 - r^2}}{\|p_{rel}\|},
\]
so that \(\phi\) represents the half-angle of the collision cone. Then, under the assumption that \(\|p_{rel}\|>r\), there exists a locally Lipschitz control law \(u(x)\) such that
\[
L_fh(x) + L_gh(x)u(x) + \frac{\partial h}{\partial t} \geq -\alpha(h(x)),
\]
for an appropriate extended class-\(\mathcal{K}\) function \(\alpha\). In other words, the candidate \(h(x,t)\) is a valid control barrier function on the set
\[
D=\{x\mid \|p_{rel}\|>r\}.
\]
\end{theorem}

\begin{proof}
Let the state of the end-effector be \(x=[p_m^T,\,v_m^T]^T\), so that the dynamics can be written in the control-affine form
\[
\dot{x} = f(x) + g(x)u,
\]
with
\[
f(x)=\begin{bmatrix} v_m \\ 0 \end{bmatrix} \quad \text{and} \quad g(x)=\begin{bmatrix} 0 \\ I_{3\times3} \end{bmatrix}.
\]
Since the obstacle moves with constant velocity, its acceleration is zero, and thus the relative dynamics are given by
\[
\dot{p}_{rel} = v_{rel},\quad \dot{v}_{rel} = u.
\]
By definition, the candidate CBF is 
\[
h(x,t)=\langle p_{rel}, v_{rel}\rangle + \|p_{rel}\|\|v_{rel}\|\cos\phi,
\]
where the function \(\cos\phi\) is determined solely by the geometry of the safety margin \(r\) and the current relative position, namely, 
\[
\cos\phi=\frac{\sqrt{\|p_{rel}\|^2-r^2}}{\|p_{rel}\|}.
\]
Differentiating \(h(x,t)\) with respect to time yields
\[
\dot{h}(x,t)=\frac{d}{dt}\langle p_{rel}, v_{rel}\rangle + \frac{d}{dt}\left(\|p_{rel}\|\|v_{rel}\|\cos\phi\right).
\]
The first term differentiates as 
\[
\frac{d}{dt}\langle p_{rel}, v_{rel}\rangle = \langle \dot{p}_{rel}, v_{rel}\rangle + \langle p_{rel}, \dot{v}_{rel}\rangle = \|v_{rel}\|^2 + \langle p_{rel}, u \rangle.
\]
The second term, which involves the product \(\|p_{rel}\|\|v_{rel}\|\cos\phi\), can be differentiated using the product and chain rules. Although the expression is somewhat lengthy, all terms depend on \(p_{rel}\) and \(v_{rel}\) (and their derivatives), and crucially, the control input \(u\) appears linearly through \(\dot{v}_{rel}=u\). Consequently, the time derivative of \(h(x,t)\) can be written in the form
\[
\dot{h}(x,t)=L_fh(x) + L_gh(x)u,
\]
where \(L_fh(x)\) collects all terms independent of \(u\) and \(L_gh(x)\) is the vector that multiplies \(u\).

Because the manipulator is assumed to be fully actuated in the task space, the term \(L_gh(x)\) is nonzero when the system is near the boundary of the safe set (i.e., when \(\|p_{rel}\|>r\)). This implies that for any state \(x\) in the domain \(D\), one can choose a control input \(u\) that satisfies
\[
L_fh(x) + L_gh(x)u \geq -\alpha\big(h(x,t)\big),
\]
for an appropriate choice of the extended class-\(\mathcal{K}\) function \(\alpha\). Therefore, by the definition of a control barrier function, \(h(x,t)\) is a valid CBF for the set \(D\).
\end{proof}

\begin{figure}
    \centering
    \includegraphics[width=0.7\linewidth]{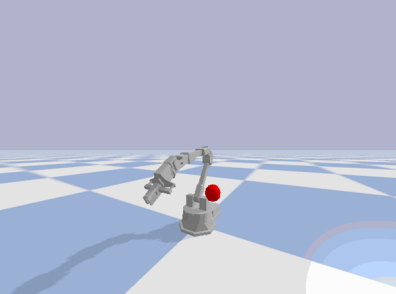}
    \caption{PyBullet Simulation of manipulator and obstacle}
    \label{fig:a1}
\end{figure}

\section{Simulation Experiments}
To evaluate the performance of the proposed safety-critical control scheme, we conducted extensive simulation experiments using PyBullet physics engine \cite{coumans2019}. In our simulation environment, a six-degree-of-freedom robotic manipulator was modeled with realistic dynamic parameters derived from standard Denavit-Hartenberg formulations. The manipulator was controlled using Cartesian impedance control, where the desired end-effector trajectory was specified along with target values for stiffness, damping, and inertia. The simulation also included dynamic obstacles modeled as moving spheres, where the safety margin was defined by a predetermined radius. A perception boundary was incorporated to simulate sensor limitations, such that the safety filter was only activated when an obstacle entered the observable region.

\begin{figure*}[htp]
    \centering
    \begin{subfigure}{.51\textwidth}
      \centering
      \includegraphics[width=.99\linewidth]{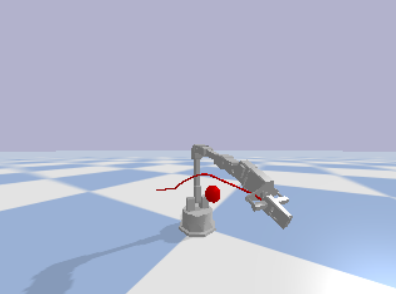}
      \caption{}
      \label{fig:sfig1}
    \end{subfigure}%
    \begin{subfigure}{.49\textwidth}
      \centering
      \includegraphics[width=.97\linewidth]{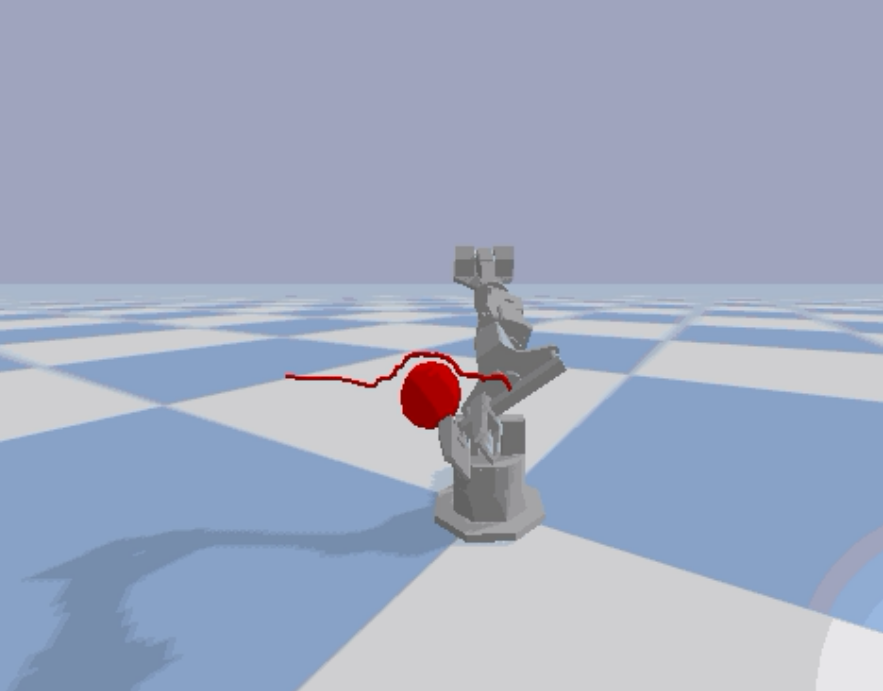}
      \caption{}
      \label{fig:sfig2}
    \end{subfigure}
    \caption{PyBullet Simulation of a Manipulator avoiding the obstacle using distance based CBF}
    \label{fig:a2}
\end{figure*}

The simulation scenarios were carefully designed to test the controller under a variety of conditions. In one scenario, a static obstacle was placed directly in the planned path of the manipulator’s end-effector. As the end-effector approached the obstacle, the QP-based safety filter modified the control input to alter the trajectory and avoid collision, while still striving to maintain the desired impedance behavior. In another scenario, an obstacle moving directly toward the manipulator was introduced. The relative velocity between the end-effector and the obstacle was computed in real time, and the QP adjusted the control commands so that the relative velocity vector was steered away from the collision cone. A further simulation scenario involved an obstacle crossing the manipulator’s path at an angle, which tested the controller’s ability to handle complex and rapidly changing dynamics. Finally, a multi-obstacle scenario was simulated in which several obstacles with varying trajectories were present, challenging the safety filter to coordinate avoidance maneuvers effectively while preserving the tracking performance of the nominal controller.

During the experiments, the QP was solved at a high frequency (approximately 240 Hz) to ensure a rapid response to dynamic changes. The impedance control gains were tuned to achieve a balance between compliant behavior and trajectory tracking, while the extended class-$\mathcal{K}$ function parameter was chosen to provide a robust safety margin. The results of the simulations indicated that the safety filter effectively prevented collisions in all tested scenarios. Moreover, the nominal trajectory was only minimally perturbed, and the additional control effort introduced by the safety filter was kept to a minimum. Detailed plots of the end-effector trajectory, obstacle paths, and evolution of the QP constraint were used to validate the theoretical predictions and demonstrate the overall robustness of the proposed approach.

\section{Results and Discussion}
The simulation experiments confirmed that the integration of Collision Cone Control Barrier Functions with Cartesian impedance control can guarantee safety without significantly compromising performance. The manipulator consistently avoided collisions even when obstacles moved unpredictably, and the modifications to the nominal control input were sufficiently subtle to maintain the desired compliance and tracking accuracy. The high update rate of the QP ensured that the system responded promptly to changes in obstacle dynamics, and the overall control effort remained within acceptable limits. The simulation outcomes provide a strong indication that the proposed method is both robust and effective, paving the way for potential experimental validation on physical robotic platforms.

\section{Conclusion and Future Work}
In this paper, we have developed a detailed and comprehensive safety-critical control framework for robotic manipulators by extending Collision Cone Control Barrier Functions to the realm of compliant manipulation. By integrating these safety constraints with Cartesian impedance control, we have formulated a QP-based safety filter that guarantees collision avoidance while preserving the nominal tracking performance. The extensive simulation experiments demonstrate that our approach is capable of handling a variety of dynamic scenarios, including static and moving obstacles as well as cluttered environments. Future work will focus on experimental validation using physical hardware and on further refining the approach to handle control input saturation and complex multi-obstacle scenarios.
In future work, we plan to investigate neural network-based control barrier functions as demonstrated in \cite{tayal2024learning, zhang2023exact}, and extend our approach by incorporating visuomotor inputs as discussed in \cite{tayal2024semi, harms2024neural} for enhanced safety and adaptability in manipulator control.

\vspace{10em}

\label{section: References}
\bibliographystyle{IEEEtran}
\bibliography{references.bib}

\end{document}